\documentclass[11pt]{article}

\usepackage[final]{acl}

\usepackage{times}
\usepackage{latexsym}
\usepackage{times}
\usepackage{latexsym}
\usepackage{titletoc}
\usepackage{tcolorbox}
\usepackage{tabularx}
\usepackage{booktabs}
\tcbuselibrary{skins,breakable}
\usepackage{listings}
\usepackage{listingsutf8}

\usepackage[T1]{fontenc}

\usepackage[utf8]{inputenc}
\usepackage{newunicodechar}
\usepackage{amssymb}
\usepackage{threeparttable}
\definecolor{Qwen17}{RGB}{245,242,255} 
\definecolor{Qwen4B}{RGB}{237,250,242} 
\usepackage{amsthm}
\newtheorem{definition}{Definition}
\newtheorem{theorem}{Theorem}
\newtcolorbox{dialogbox}{
  colback=gray!10!white, colframe=black, sharp corners,
  boxrule=0.5mm, top=10pt, bottom=10pt, left=10pt, right=10pt, breakable
}
\usepackage{textcomp} 
\usepackage{listings}
\usepackage{upquote}  

\newunicodechar{✓}{\checkmark}
\newunicodechar{✗}{\ding{55}}

\usepackage{microtype}
\usepackage{enumitem}
\setlist{nosep}

\usepackage{inconsolata}

\usepackage{graphicx}
\usepackage{amsmath} 
\usepackage{amsmath, amssymb} 
\usepackage{tikz}
\usetikzlibrary{positioning} 
\usepackage{hyperref}
\usepackage{multirow}
\usepackage[table,xcdraw]{xcolor}
\usepackage[dvipsnames]{xcolor}
\tcbuselibrary{skins,breakable,listings}
\usepackage{listings}
\usepackage{listingsutf8}
\usepackage{pifont} 
\usepackage{subcaption}
\usepackage{caption}
\usepackage{booktabs}
\usepackage{fontawesome5}
\usepackage{algorithm}     
\usepackage{algpseudocode} 
\usepackage{float}

\usepackage[T1]{fontenc}

\usepackage[utf8]{inputenc}

\usepackage{microtype}

\usepackage{inconsolata}
\usepackage{adjustbox}

\title{SQL-ASTRA: Alleviating Sparse Feedback in Agentic SQL via Column-Set Matching and Trajectory Aggregation}

\author{
  \textbf{Long Li}$^{1,\dagger}$ \quad \textbf{Zhijian Zhou}$^{2,6,\dagger}$ \quad \textbf{Jiangxuan Long}$^{3}$ \quad \textbf{Peiyang Liu}$^{4}$ \\
  \textbf{Weidi Xu}$^{5}$ \quad \textbf{Zhe Wang}$^{1}$ \quad \textbf{Shirui Pan}$^{1,*}$ \quad \textbf{Chao Qu}$^{2,7,}\thanks{Co-corresponding authors. $^{\dagger}$Equal contribution.}$ \\
  $^{1}$Griffith University, Brisbane, Australia \quad
  $^{2}$Fudan University, Shanghai, China \\
  $^{3}$The University of Hong Kong, Hong Kong, China \quad
  $^{4}$Peking University, Beijing, China \\
  $^{5}$InFly, China \quad
  $^{6}$Shanghai Innovation Institute, China \\
  $^{7}$Shanghai Academy of Artificial Intelligence for Science, China \\
  \texttt{long.li@griffithuni.edu.au} \\
}

\begin{document}
\maketitle
\begin{abstract}
Agentic Reinforcement Learning (RL) shows promise for complex tasks, but Text-to-SQL remains mostly restricted to single-turn paradigms. A primary bottleneck is the credit assignment problem. In traditional paradigms, rewards are determined solely by the final-turn feedback, which ignores the intermediate process and leads to ambiguous credit evaluation. To address this, we propose Agentic SQL, a framework featuring a universal two-tiered reward mechanism designed to provide effective trajectory-level evaluation and dense step-level signals. First, we introduce Aggregated Trajectory Reward (ATR) to resolve multi-turn credit assignment. Using an asymmetric transition matrix, ATR aggregates process-oriented scores to incentivize continuous improvement. Leveraging Lyapunov stability theory, we prove ATR acts as an energy dissipation operator, guaranteeing a cycle-free policy and monotonic convergence. Second, Column-Set Matching Reward (CSMR) provides immediate step-level rewards to mitigate sparsity. By executing queries at each turn, CSMR converts binary (0/1) feedback into dense $[0, 1]$ signals based on partial correctness. Evaluations on BIRD show a 5\% gain over binary-reward GRPO. Notably, our approach outperforms SOTA Arctic-Text2SQL-R1-7B on BIRD and Spider 2.0 using identical models, propelling Text-to-SQL toward a robust multi-turn agent paradigm.
\end{abstract}

\section{Introduction}

In recent years, Agentic Reinforcement Learning (RL) has garnered significant attention~\cite{zhang2025landscape,zhu2025medeyes}, empowering Large Language Models (LLMs) to engage in multi-turn interactions with environments to accomplish more complex tasks, such as deep research~\cite{openai2025o3,xu2025comprehensive,huang2025deep}, Web searching~\cite{jin2025_searchR1,guo2025tree}, and code execution~\cite{jiang2025verltoolholisticagenticreinforcement,feng2025retool}. In most scenarios, the capabilities of these methods are enhanced through Reinforcement Learning with Verifiable Rewards (RLVR)~\cite{zhang2025spft,zhou2026look,tan2025bottom}, where feedback is provided based on the correctness of the final outcome~\cite{deepseekai2025deepseekr1incentivizingreasoningcapability,li2025choice,cai2025reinforcement,cai2026vi}.

\begin{figure*}[htb!]
    \centering
    \includegraphics[width=\textwidth]{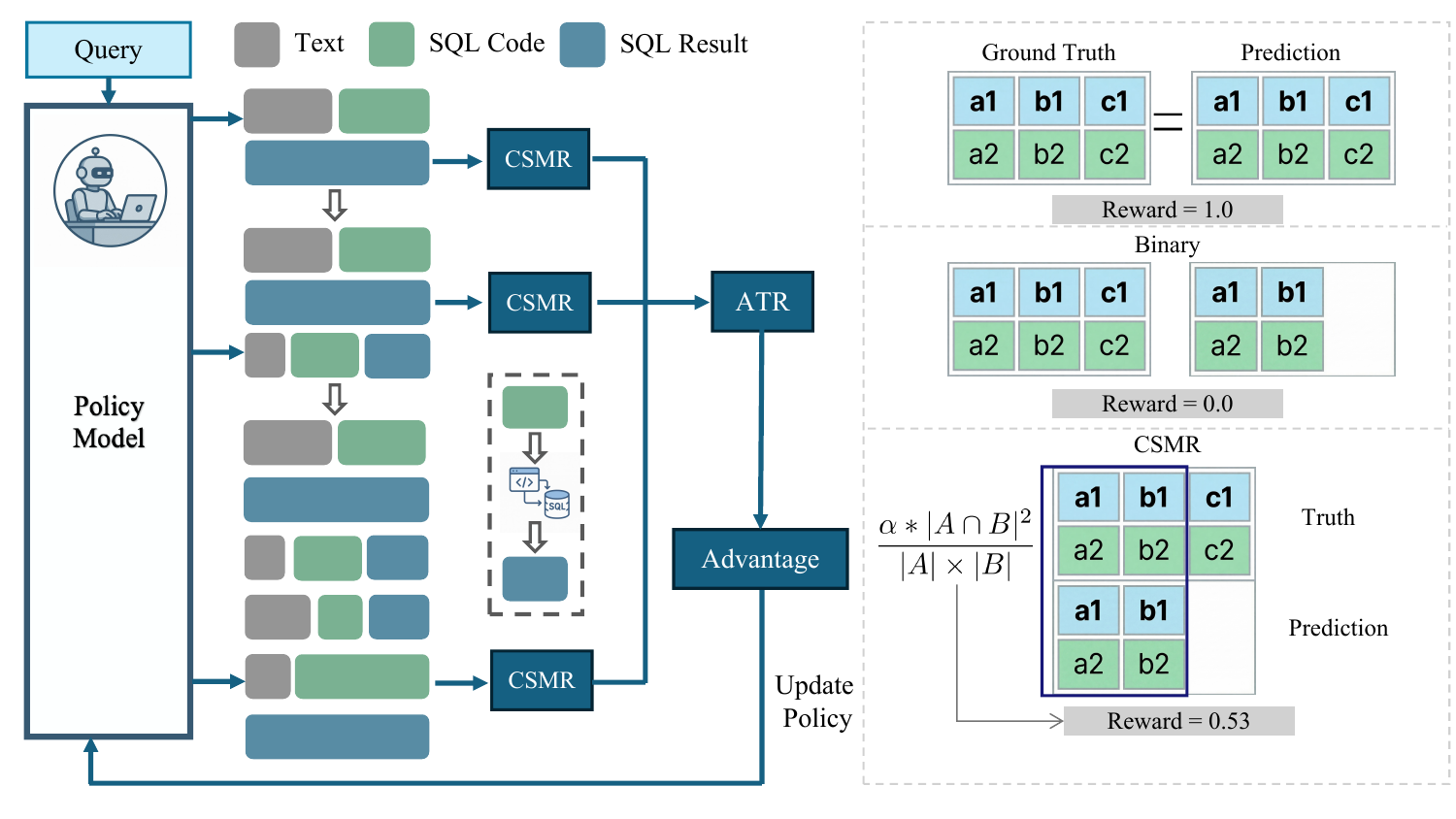}
    \caption{The framework of Agentic SQL. CSMR captures rich signals from error cases, whereas traditional binary rewards overlook this information.}
    \label{fig:main}
\end{figure*}

However, when applied to complex tasks requiring exploratory reasoning, Agentic RL faces three core challenges: \textbf{(1) Paradigm Constraint:} While agents are designed for multi-turn interaction, most existing work in specific domains like \textbf{Text-to-SQL} remains confined to a single-turn, static generation paradigm \cite{ma2025sql,yao2025arctic,ali2025state}. This framework fails to reflect the dynamic process of human data analysts who use multiple tentative queries to gather context and refine their strategy \cite{kim2024flex,furst2024evaluating}, which severely limits the model's ability to solve complex, real-world problems \cite{huo2025bird,leispider}. \textbf{(2) Credit Assignment:} In multi-turn trajectories, evaluation signals typically rely exclusively on \textbf{final-turn feedback} \cite{guo2025mtsql,feng2025retool,ding2026prpoaligningprocessreward,pang2026ticgrpoprovableefficientoptimization}. This ``all-or-nothing'' approach treats the interaction sequence as a black box, introducing a pervasive credit assignment problem where the agent cannot distinguish which intermediate steps contributed to the final outcome. \textbf{(3) Micro-level Reward Sparsity:} Even when step-level feedback is available, it is often restricted to coarse, \textbf{binary (0/1) signals} based on execution success \cite{ma2025sql,lei2025reasoning,yao2025arctic}. Such sparse feedback ignores the rich information in ``partially correct'' queries, providing insufficient granular guidance and drastically restricting the efficiency and robustness of RL training.

To systematically address these challenges, we propose an innovative methodology:

\begin{enumerate}[label=(\arabic*), wide=0pt]
    \item We construct a multi-turn interactive framework that enables agents to iteratively gather context and refine SQL queries through dynamic database interaction, effectively overcoming the single-turn limitations.
    
    \item To address the reliance on final-turn feedback, we design the \textbf{Aggregated Trajectory Reward (ATR)}. Unlike traditional methods, ATR utilizes an \textbf{Asymmetric Transition Matrix} to aggregate signals across the entire reasoning path, explicitly incentivizing strategies that exhibit continuous and monotonic improvement. Crucially, we provide a \textbf{rigorous Theoretical Guarantee}: by modeling the reasoning process as a dynamical system, we prove that ATR acts as an energy dissipation operator, mathematically ensuring a cycle-free policy and monotonic convergence based on Lyapunov stability theory.

    \item We propose the \textbf{Column-Set Matching Reward (CSMR)} as an immediate, \textbf{dense step-level reward}. By evaluating \textit{partial correctness} via column value-set normalization, CSMR transforms $0/1$ outcomes into granular signals within the $[0, 1]$ range. This provides precise, step-by-step guidance and serves as the foundational input for ATR trajectory aggregation.
\end{enumerate}

\section{Methodology and Convergence Analysis}

\subsection{Agentic SQL Framework}
\label{sec:problem_formulation}

Traditional Reinforcement Learning (RL) approaches for Large Language Models (LLMs) are mostly confined to a single-turn interaction setting, aiming to maximize the expected reward $J_{\text{single-turn}}(\theta)$ for a single response.
$$
J_{\text{single-turn}}(\theta) = \mathbb{E}_{s \sim \mathcal{D}, a \sim \pi_\theta(\cdot|s)} [R(s, a)].
$$
However, this paradigm is insufficient for complex, exploratory tasks like Text-to-SQL, which require the agent to engage in multi-step interaction, strategy iteration, and error recovery.

To accurately model this dynamic, we formalize the interactive Text-to-SQL task as a Finite-Horizon Markov Decision Process (MDP) defined by the tuple $\mathcal{M} = \langle \mathcal{S}, \mathcal{A}, \mathcal{P}, \mathcal{R} \rangle$. An agent learns a policy $\pi_\theta$ by generating interaction trajectories $\tau = (s_0, a_0, \dots, s_K)$ to maximize the expected cumulative reward based on our process-based reward function $\mathcal{R}$:
\begin{equation}
\label{eq:objective}
J_{\text{multi-turn}}(\theta) = \mathbb{E}_{\tau \sim \pi_\theta} \left[ \sum_{t=0}^{K-1} \mathcal{R}(s_t, a_t) \right]
\end{equation}

The overall framework is illustrated in Figure~\ref{fig:main}. In each iteration, after generating the SQL code, the agent interacts with the database to retrieve the execution result and subsequently employs CSMR to derive a dense reward. Upon completion of the multi-turn generation, the ATR mechanism aggregates the step-wise CSMRs to formulate the advantage for the entire trajectory.

\subsection{Reward Design}

\begin{algorithm}[H] 
\small
\begin{minipage}{\linewidth} 
\caption{CSMR (Column-Set Matching Reward) with Perfect Match Check}
\label{alg:csmr_with_check}
\begin{algorithmic}[1]
\Require{Gold result table $G$, Predicted result table $P$, Scaling factor $\alpha$ (e.g., 0.8)}
\Ensure{Reward score $R_{\text{CSMR}}$}

\State \Comment{\textbf{Step 0: Check for Perfect Match}}
\If{\Call{IsPerfectMatch}{G, P}} \Comment{e.g., \texttt{set(G) == set(P)}}
   \State \Return 1.0
\EndIf
\Statex

\State \Comment{\textbf{Step 1: Column Value-Set Extraction (for partial score)}}
\State $G' \gets \Call{DeduplicateRows}{G}$
\State $P' \gets \Call{DeduplicateRows}{P}$
\State $N_c^G \gets \Call{GetNumColumns}{G'}$
\State $N_c^P \gets \Call{GetNumColumns}{P'}$
\State $D_G \gets \emptyset$ \Comment{Initialize set of gold column-sets}
\State $D_P \gets \emptyset$ \Comment{Initialize set of predicted column-sets}

\For{$i \gets 1$ \textbf{to} $N_c^G$}
 \State $C_i^G \gets \Call{GetUniqueValuesInColumn}{G', i}$
 \State $D_G \gets D_G \cup \{C_i^G\}$
\EndFor
\For{$j \gets 1$ \textbf{to} $N_c^P$}
 \State $C_j^P \gets \Call{GetUniqueValuesInColumn}{P', j}$
 \State $D_P \gets D_P \cup \{C_j^P\}$
\EndFor
\Statex

\State \Comment{\textbf{Step 2: Column-Set Matching}}
\State $M \gets 0$ \Comment{Initialize match count}
\For{\textbf{each} $C^G \in D_G$}
  \If{\textbf{exists} $C^P \in D_P$ \textbf{s.t.} $C^G = C^P$}
  \State $M \gets M + 1$
  \EndIf
\EndFor
\Statex

\State \Comment{\textbf{Step 3: Compute Score}}
\State $N_{\text{prod}} \gets N_c^G \times N_c^P$
\If{$N_{\text{prod}} = 0$}
 \State $R_{\text{CSMR}} \gets 0.0$
\Else
 \State $R_{\text{CSMR}} \gets \frac{M^2}{N_{\text{prod}}} \times \alpha$
\EndIf
\Statex

\State \Return $R_{\text{CSMR}}$
\end{algorithmic}
\end{minipage} 
\end{algorithm}

\subsubsection{Column-Set Matching Reward}
As shown on the right side of Figure~\ref{fig:main}, in Text-to-SQL tasks, traditional reward mechanisms typically employ a sparse binary 0/1 signal. This mechanism requires the execution result of a predicted query, $P$, to \textit{perfectly match} the execution result of the gold answer, $G$. Any slight discrepancy (e.g., row ordering, partially missing data) results in a reward of 0. This sparse reward signal significantly hinders RL convergence, as it ignores the substantial ``partially correct'' information that predicted results may contain.

To address this, we propose Column-Set Matching Reward (CSMR), a novel, denser reward function designed to measure the meaningful \textit{structural overlap} between $P$ and $G$. Our core idea is to move away from comparing \textit{rows} (tuples) and instead compare the \textit{sets of values} within each \textit{column}. Specifically, our reward calculation process is detailed in Algorithm~\ref{alg:csmr_with_check}.

The CSMR mechanism, using product-based normalization, effectively captures \textit{partial correctness}. It provides a dense reward signal for matching column value-sets, even with incorrect row composition, thereby alleviating the sparsity problem of binary 0/1 rewards.

We note that CSMR is insensitive to row-level combinations. For example, $P$ = \texttt{[[a1, b2], [a2, b1]]} would perfectly match the column-sets of $G$ = \texttt{[[a1, b1], [a2, b2]]}. To address this, we introduce a scaling factor $\alpha < 1.0$ (e.g., 0.8). This factor caps the maximum reward from CSMR at $\alpha$, intentionally distinguishing these `pseudo-perfect' (column-only) matches from truly perfect row-level matches.

\subsection{Aggregated Trajectory Reward and Theoretical Guarantee}
While CSMR provides spatial density, it is insufficient for evaluating the temporal quality of a complete multi-step trajectory $\tau$. In a traditional RL setup, one might employ a \textit{Step-wise Update} strategy, which assigns and backpropagates reward signals $r_t$ immediately after each interaction turn. However, this naive approach suffers from high gradient variance due to \textbf{inconsistent signal scaling} across heterogeneous database schemas. Furthermore, it is susceptible to \textbf{reward hacking}—a phenomenon where agents learn to exploit the reward function through intentional oscillations. 

To overcome these limitations, we design the \textbf{Aggregated Trajectory Reward (ATR)} $R_{\text{ATR}}$. Unlike step-wise reinforcement, ATR is a single scalar reward provided only at episode termination to explicitly incentivize \textbf{continuous improvement} while suppressing oscillatory behaviors. We provide a detailed ablation study comparing ATR with the \textit{Step-wise Update} in Section~\ref{exp:atr}.

\paragraph{ATR Calculation}
The ATR aggregates the trajectory's utility by weighing both the direction and magnitude of semantic changes. The total reward is computed as:

The step-wise transition function $\Psi$ is derived from an asymmetric matrix $\mathcal{M}$ scaled by the gradient magnitude $|\Delta R_t|$, where $\Delta R_t = R_t - R_{t-1}$.
The core inductive bias is encoded in $\mathcal{M}$, where each element represents a specific semantic behavior:
\begin{equation}
\label{eq:atr_all}
\begin{split}
    \mathcal{M} &= \begin{pmatrix} 
        R_{\text{Low} \to \text{Low}} & R_{\text{Low} \to \text{High}} \\ 
        R_{\text{High} \to \text{Low}} & R_{\text{High} \to \text{High}} 
    \end{pmatrix} \\
    &= \begin{pmatrix} 
        -0.0 & +1.0 \\ 
        -1.5 & 0.0 
    \end{pmatrix} \\
    \Psi(R_{t-1}, R_t) &= \mathcal{M}_{s_{t-1}, s_t} \cdot 
    \begin{cases} 
        |\Delta R_t| & \text{if } s_{t-1} \neq s_t \\
        1.0 & \text{if } s_{t-1} = s_t
    \end{cases}
\end{split}
\end{equation}
\noindent The binary states $(s_{t-1}, s_t)$ are determined dynamically by the gradient and a stagnation threshold $\tau$:
\begin{enumerate}[label=(\roman*), leftmargin=*]
    \item \textbf{Transition ($s_{t-1} \neq s_t$):} A positive gradient $\Delta R_t > 0$ triggers an improvement transition, while $\Delta R_t < 0$ triggers degradation.
    \item \textbf{Stagnation ($s_{t-1} = s_t$):} When $\Delta R_t \approx 0$, the state is sustained as High if $R_t > \tau$, otherwise it remains Low.
\end{enumerate}
This design ensures that the agent is rewarded proportionally to the \textit{extent} of improvement ($|\Delta R_t|$), while the asymmetric $\mathcal{M}$ (with $|R_{\text{High} \to \text{Low}}| > |R_{\text{Low} \to \text{High}}|$) ensures strict dissipativity.

\paragraph{Theoretical Foundation: Lyapunov-Guided Convergence}

We demonstrate that the ATR mechanism is not merely heuristic, but acts as a necessary energy dissipation operator, guaranteeing asymptotic stability (i.e., convergence to the correct SQL query) from the perspective of Lyapunov stability theory~\cite{pmlr-v80-taylor18a}.

\textbf{1. Semantic Error Energy (Lyapunov Candidate Function):}
We transform the CSMR score $\Phi(s_t) = R_{\text{CSMR}}(s_t)$ into the \textbf{Semantic Error Energy} $V(s_t)$:
\begin{equation}
    V(s_t) = 1 - \Phi(s_t)
\end{equation}
This function satisfies the fundamental Lyapunov properties, where $V(s) = 0$ if and only if the system reaches the equilibrium point $s^*$ (the correct SQL).

\textbf{2. Stability Criterion and ATR's Dissipative Role:}
Asymptotic stability requires the energy sequence to be strictly decreasing, $\Delta V(s_t) = V(s_{t+1}) - V(s_t) < 0$.
Our $R_{\text{step}}$ is explicitly designed to approximate the negation of the orbital derivative:
\begin{equation}
    R_{\text{step}}(s_t, s_{t+1}) \propto -\nabla V(s_t) \approx \Phi(s_{t+1}) - \Phi(s_t)
\end{equation}
Maximizing the cumulative ATR implicitly optimizes the policy $\pi_\theta$ to maximize the rate of energy dissipation, accelerating convergence towards $s^*$.

\textbf{3. Necessity of Asymmetric Shaping: Eliminating Limit Cycles}

Multi-turn agents often suffer from \textit{limit cycles}, oscillating between suboptimal states. We prove an \textbf{Asymmetric Transition Matrix $\mathcal{M}$} is necessary to eliminate such cycles. While symmetric rewards satisfy $\sum_{t \in \mathcal{C}} R_t = 0$ for any cycle $\mathcal{C}$---yielding a merely \textit{Lagrange stable}~\cite{leine2010historical,zhang2018multiple} and oscillatory system---ATR imposes $|R_{\text{High} \to \text{Low}}| > |R_{\text{Low} \to \text{High}}|$, ensuring a net reward loss ($\sum_{t \in \mathcal{C}} R_t < 0$). This transforms the reasoning into a \textbf{strictly dissipative system}, forcing $\pi^*$ to be cycle-free along a monotonic energy-descent path (Appendix~\ref{app:proof}).

\textbf{4. Robustness via Quantized Feedback:}
The threshold $\tau$ in the ATR calculation acts as a \textbf{quantized feedback filter}, akin to techniques in Sliding Mode Control.
$$
S_t = \mathbb{I}(\Phi(s_t) > \tau)
$$
Given the stochastic nature of LLMs, the continuous potential $\Phi(s_t)$ may contain high-frequency noise. By discretizing the state transitions, we ensure that the Lyapunov derivative $\Delta V$ drives policy updates only when a \textbf{significant semantic phase transition} occurs, thereby enhancing the robustness of the learning process.

\subsection{GRPO Training with Tool Masking}

We use the GRPO algorithm to optimize our policy model~\cite{deepseekai2025deepseekr1incentivizingreasoningcapability}. For a specific question-answer pair $(q,a)$, GRPO's underlying behavior policy, $\pi_{\theta_\text{old}}$, generates a group of $G$ individual responses, denoted as $\{ o_i\}_{i=1}^G$. The advantage for the $i$-th response within this ensemble is then calculated by normalizing the rewards specific to that group, $\{ R_{\text{ATR}, i} \}_{i=1}^G$:
\begin{equation}
\hat{A}_{i,t} = \frac{R_{\text{ATR}, i} - \text{mean}(\{R_{\text{ATR}, j}\}_{j=1}^G)}{\text{std}(\{R_{\text{ATR}, j}\}_{j=1}^G)}.
\end{equation}
Thus, every token in the trajectory uses this normalized reward as its advantage. Additionally, we introduce a binary mask,
\begin{equation}
M_{i,t} =
\begin{cases}
1 & \text{if } o_{i,t} \text{ is reasoning token} \\
0 & \text{if } o_{i,t} \text{ is execution token}
\end{cases} 
\end{equation}
This loss masking ensures the model focuses on learning the reasoning process. We define the clipped surrogate objective as $\rho^{\text{clip}}_{i,t}(\theta) = \text{clip} \left( \rho_{i,t}(\theta), 1 - \varepsilon, 1 + \varepsilon \right)$. The optimized GRPO loss is:
\begin{equation}
\begin{aligned}
\mathcal{L}_{\text{group}}(\theta) = & \frac{1}{G}\sum_{i=1}^{G} \frac{1}{|o_i|}\sum_{t=1}^{|o_i|} \\
& \quad M_{i,t} \min \left( \rho_{i,t}(\theta) \hat{A}_{i,t}, \rho^{\text{clip}}_{i,t}(\theta) \hat{A}_{i,t} \right) \\[1ex] 
\mathcal{L}(\theta) = & -\mathbb{E}_{\substack{(q,a)\sim \mathcal{D} \\ {o_i}\sim 
\pi_{\theta_{\text{old}}}(\cdot\mid q)}} \Big[ \mathcal{L}_{\text{group}}(\theta) \Big]
\end{aligned}
\end{equation}
where $G$ is the number of sampled trajectories per group, and $\rho_{i,t}(\theta)$=$\frac{\pi_{\theta}(o_{i,t}\mid q,o_{i,<t})}{\pi_{\theta_{\mathrm{old}}}(o_{i,t}\mid q,o_{i,<t})}$ is the per-token importance ratio.

\section{Experiment}
\subsection{Setting}
Our experiments consist of two distinct sets of trials. The first set utilizes \textbf{Qwen2.5-7B-Instruct} as the base model, which undergoes direct RL training without a cold-start phase to demonstrate the viability of our approach. This configuration is compared with other methods employing Qwen2.5-Coder\footnote{Qwen2.5-Coder is overly specialized 
due to its training regime, leading to insufficient exploratory capabilities that render it unsuitable for this experiment.}. The second set involves the \textbf{OmniSQL} model, which requires a Format-6k fine-tuning step to acquire the tool-calling format; this is then compared with other methods based on the same model. All experiments are conducted on 32 NVIDIA A800-80G GPUs.

In terms of evaluation, we primarily utilize BIRD-Dev~\cite{li2024can} and Spider~\cite{yu2018spider} to assess general SQL capabilities. Furthermore, we evaluate the model's Agent capabilities using the challenging enterprise-grade Spider-2.0 dataset~\cite{lei2024spider}. Please refer to Appendix~\ref{append:parameters} for comprehensive training details.

\begin{table*}[h]
\centering
\begin{threeparttable}
\setlength{\tabcolsep}{2.5pt} 
\renewcommand{\arraystretch}{1.15}

\resizebox{\linewidth}{!}{%
\begin{tabular}{lccccc}
\toprule
\textbf{Model} & \textbf{RL Data} & \textbf{Base Model} & \textbf{Bird (Greedy)} & \textbf{Spider-Dev (Greedy)} &\textbf{Spider2.0-Lite (Mj@8)} \\
\midrule
GPT-4o & — & — & 64.4 & 80.4 & 15.6 \\
Deepseek-V3 & — & — & 62.5 & 78.3 & 15.6 \\
Qwen2.5-7B-Instruct~\cite{qwen2.5} & — & - & 47.5 & 74.7 & 5.1 \\
Qwen2.5-Coder-7B~\cite{qwen2.5} & — & - & 58.2 & 77.4 & 2.2 \\
SQL-R1 + Qwen2.5-Coder-7B~\cite{ma2025sql} & SynSQL-Complex-5K & Qwen2.5-Coder-7B & 63.1 & 84.5  & - \\
Reasoning-SQL-7B~\cite{ma2025sql} & BIRD & Qwen2.5-Coder-7B & 64.0   &78.7   & - \\
OmniSQL-7B~\cite{li2025omnisql} & — & Qwen2.5-Coder-7B & 64.1 & 85.0 & 12.6 \\
SQL-R1 + OmniSQL-7B~\cite{ma2025sql} & SynSQL-Complex-5K & OmniSQL-7B &  66.6 & 87.6  & 14.8 \\
Arctic-Text2SQL-R1-7B~\cite{yao2025arctic} & BIRD + SPIDER + Gretel* & OmniSQL-7B & 67.6   & -   & 15.6 \\
MTIR-SQL-4B~\cite{xu2025mtirsqlmultiturntoolintegratedreasoning} & BIRD + SPIDER & Qwen3-4B & 63.1   & 82.4   & - \\
\midrule
\multicolumn{6}{l}{\textit{Our Results}} \\
\multicolumn{6}{c}{\textit{Binary Reward}} \\
\rowcolor{Qwen17} Single-Turn GRPO & BIRD & Qwen2.5-7B-Instruct & 58.5   & 79.2   & 6.6 \\
\rowcolor{Qwen17} Agentic SQL  & BIRD & Qwen2.5-7B-Instruct & 61.1   &80.7   & 5.9 \\
\rowcolor{Qwen17} Agentic SQL + ATR  & BIRD & Qwen2.5-7B-Instruct & 63.6   &81.4   & 7.4 \\

\multicolumn{6}{c}{\textit{CSMR Reward}} \\
\rowcolor{Qwen17} Single-Turn GRPO + CSMR & BIRD & Qwen2.5-7B-Instruct & 59.4   & 79.3   & 6.6 \\
\rowcolor{Qwen17} Agentic SQL + CSMR & BIRD & Qwen2.5-7B-Instruct & 62.5   &81.1   & 8.1 \\
\rowcolor{Qwen17} Agentic SQL + CSMR + ATR & BIRD & Qwen2.5-7B-Instruct & \textbf{64.2}   & \textbf{82.9}   & 8.8  \\

\multicolumn{6}{c}{\textit{Ablation Study on ATR Design}} \\
\rowcolor{Qwen17} ATR & BIRD & Qwen2.5-7B-Instruct & 64.2   & 82.9   & 8.8  \\
\quad w/ Symmetric Matrix & BIRD & Qwen2.5-7B-Instruct & 60.1 & 80.2 & 7.4 \\
\quad w/ Step-wise Update & BIRD & Qwen2.5-7B-Instruct & 61.3 & 82.0 & 8.1 \\

\multicolumn{6}{c}{\textit{Cold-Start Model}} \\
\rowcolor{Qwen17} Single-Turn GRPO & BIRD + Spider & OmniSQL-7B & 67.4   & 87.4   & 14.8 \\
\rowcolor{Qwen17} Agentic SQL + CSMR + ATR & BIRD + Spider & OmniSQL-7B & \textbf{69.1}   & \textbf{88.8}   & 17.7 \\

\bottomrule
\end{tabular}
}
\vspace{-0in}
\caption{For Bird and Spider, we consistently use Greedy decoding, while for Spider 2.0, we use 8-time sampling followed by majority voting on the answers. *Gretel indicates Gretel-Synth-Filtered.}
\label{tab:main_results}
\end{threeparttable}
\end{table*}

\subsection{Results}
\paragraph{Baseline}
Our primary baseline is the performance of the same model in a \textit{single-turn SQL} setting. We also include strong proprietary models like GPT-4o~\cite{achiam2023gpt} and DeepSeek-V3~\cite{guo2024deepseek}, along with several other powerful models based on RL algorithms for SQL tasks, such as strong performers in the single-turn setting: \textbf{OmniSQL}, \textbf{SQL-R1}, \textbf{Reasoning-SQL} and \textbf{Arctic-R1}, as well as the multi-turn Agent framework \textbf{MTIR-SQL}.

\paragraph{Main Results}

In the Qwen2.5-7B-Instruct experiments, our framework demonstrates excellent performance without requiring a cold start. By framing Actions as parsed signals and embedding a Rethink mechanism into $\mathcal{P}$, the agent learns to generate and utilize SQL results within one RL loop. Consequently, $\pi_\theta$ acquires complex reasoning chains by maximizing $J_{\text{multi-turn}}(\theta)$, enabling RL training to start from scratch. Our experiments also reveal that this approach is highly dependent on the chosen base model retaining a certain level of exploration capability; for instance, it could not be directly initiated on Qwen2.5-Coder, which provides significant inspiration for the future training of agents.

Our main results are presented in Table~\ref{tab:main_results}. 
From the experiments, we can observe that the effectiveness of Agentic SQL is highly significant.
It outperforms the \textit{single-turn binary (0/1) GRPO} method by \textbf{5.7\%} on the BIRD dataset and \textbf{3.7\%} on the Spider dataset, respectively. In comparison to other methods, although the Qwen2.5-7B-Instruct base model we used has weaker foundational capabilities than Qwen2.5-Coder-7B, after being trained with Agentic SQL, it achieves superior performance. Compared to Reasoning-SQL-7B, which uses the same training dataset, our method completely surpasses it in both in-domain and OOD evaluations. To demonstrate the effectiveness of our method on a more powerful model, we utilized the \textbf{OmniSQL-7B} model for agent reinforcement learning generation. It surpassed the single-turn $\mathbf{0/1}$ reward GRPO on both the BIRD and Spider datasets. Furthermore, it also outperformed SQL-R1 and Arctic-Text2SQL-R1-7B, which are based on OmniSQL-7B, achieving improvements of 2.5\% and 1.5\% respectively on the BIRD dataset.

\paragraph{Result in Spider2.0}

Evaluation on Spider 2.0 typically requires the introduction of more complex workflows for the model to generate the final answer in multiple steps. We demonstrated an even higher level of excellence on the larger and more practical Spider 2.0 dataset. While a series of models using only 0/1 rewards generally converge around 15\%, our model achieves 17.7\%. This is because, by integrating the CSMR + ATR mechanism, the model can receive more dense signals and learn to move towards better directions in each call, progressively solving the problem.

\subsection{Ablation Study}

Our proposed approach primarily encompasses three core components: the Agentic SQL Framework, CSMR, and ATR. We conducted a thorough ablation study on these three modules sequentially. Table~\ref{tab:main_results} demonstrates that CSMR consistently outperforms the conventional binary reward across all experimental settings, thoroughly validating the effectiveness of its fine-grained reward mechanism. Furthermore, the Agentic SQL Framework proves to be the most significant contributor to performance enhancement, yielding a performance gain of nearly 3\% on the Bird dataset. Finally, the ATR Module effectively enhances the signal density within multi-turn conversations, ensuring that sufficient and effective feedback signals are obtained for each generation step.

\begin{figure*}[t]
\begin{minipage}[c]{0.3\linewidth}
    \centering
    \adjustbox{valign=c}{
        \small
        \begin{tabular}{| l | r |}
        \hline
        \textbf{Input Distribution} & \textbf{ATR Reward} \\
        \hline
        \texttt{[1]} & 1.00 \\ \texttt{[0, 1]} & 0.90 \\ \texttt{[0, 0, 1]} & 0.80 \\
        \texttt{[1, 1]} & 0.70 \\ \texttt{[0, 1, 1]} & 0.60 \\ \texttt{[1, 1, 1]} & 0.40 \\
        \texttt{[1, 0, 1]} & 0.30 \\ \texttt{[0]} & 0.00 \\ \texttt{[0, 0]} & -0.10 \\
        \texttt{[0, 0, 0]} & -0.20 \\ \texttt{[1, 0]} & -0.60 \\ \texttt{[0, 1, 0]} & -0.70 \\
        \texttt{[1, 0, 0]} & -0.70 \\ \texttt{[1, 1, 0]} & -0.90 \\
        \hline
        \end{tabular}
    } 
\end{minipage}
\hfill
\begin{minipage}[c]{0.7\linewidth}
    \centering
    \adjustbox{valign=c}{
        \includegraphics[width=0.9\linewidth]{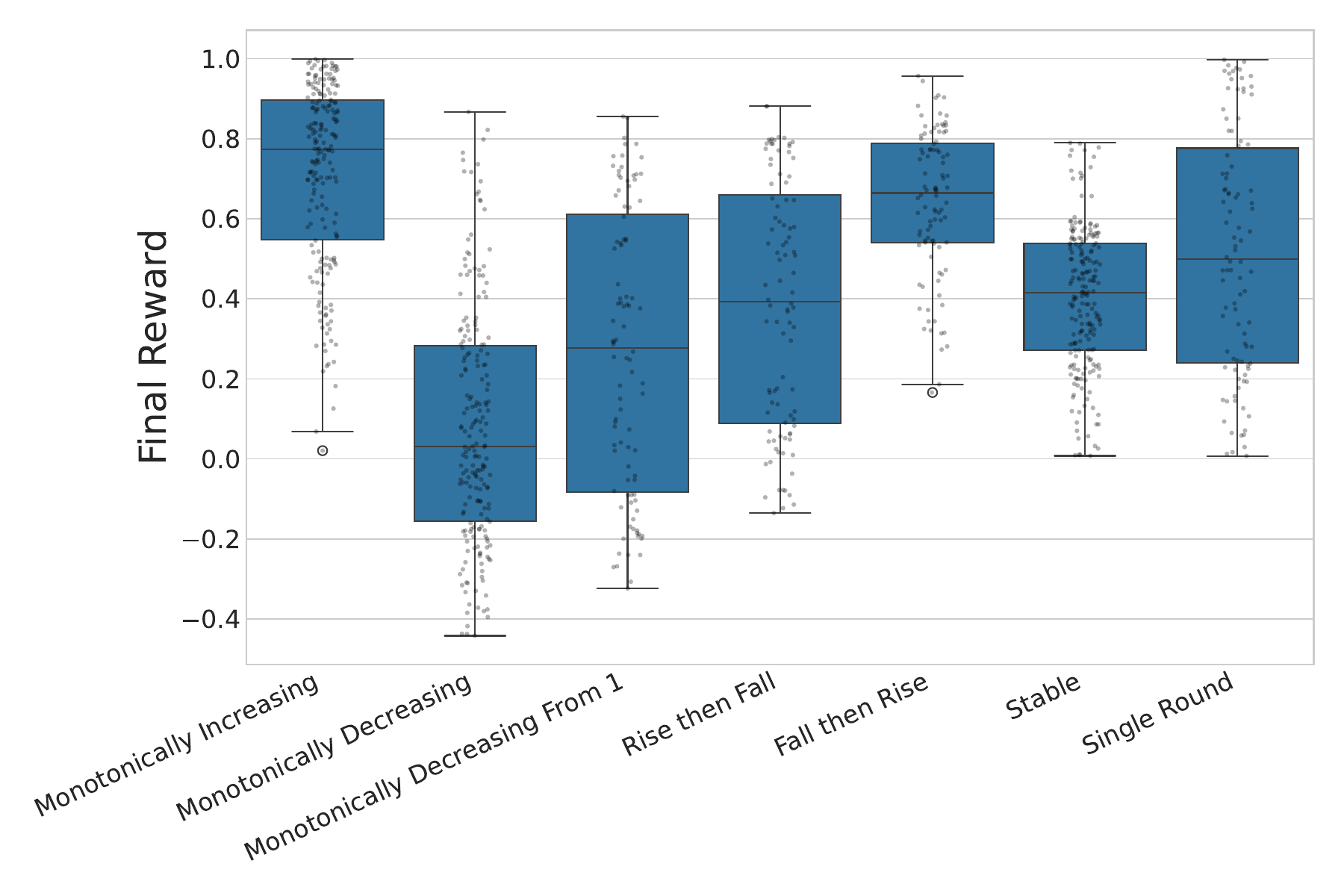}
    } 
\end{minipage}
\caption{ATR reward analysis, showing the ATR on 0/1 reward (left) and the dense reward distribution by strategy type (right).}
\label{ATR}
\end{figure*}

\subsubsection{Ablation Study on ATR Design}
\label{exp:atr}

We present the validation results regarding the necessity of each ATR component in Table~\ref{tab:main_results}.

\noindent\textbf{Necessity of Asymmetric Shaping} We compared our asymmetric design with a symmetric setting ($|R_{\text{High}\to\text{Low}}| = |R_{\text{Low}\to\text{High}}|$). We observed that the symmetric setting causes the model to engage in repetitive generation, resulting in numerous unnecessary loops and degrading sample learning efficiency. This confirms that asymmetric penalties are crucial for eliminating limit cycles and enforcing the strict dissipativity condition.

\noindent\textbf{Advantage of Trajectory Aggregation} In this variant, we assigned the step-wise reward directly as the advantage to the corresponding tokens and removed the group normalization operation. The performance drop observed highlights the efficacy of our aggregated approach in mitigating the credit assignment problem. Under this setting, the prefix context varies at each turn, making it impossible to perform group normalization based on the same original question, thereby compromising the effectiveness of the advantage estimation.

\subsubsection{Ablation Study on the Scaling Factor $\alpha$ in CSMR}

Sensitivity analysis (Table~\ref{tab:alpha_aba}) shows negligible performance variance, confirming CSMR's robustness to $\alpha$. This factor distinguishes ``perfect'' matches from \textbf{``pseudo-perfect''} ones—where column values match but row compositions differ (e.g., $P = \texttt{[[a1, b2], [a2, b1]]}$ vs. $G = \texttt{[[a1, b1], [a2, b2]]}$). By penalizing these subtle mismatches, $\alpha$ prevents false-positive reinforcement and ensures precise semantic alignment.

\section{Analyze}
\subsection{The Role of CSMR}
\begin{figure}[h]
    \centering
    \includegraphics[width=0.99\columnwidth]{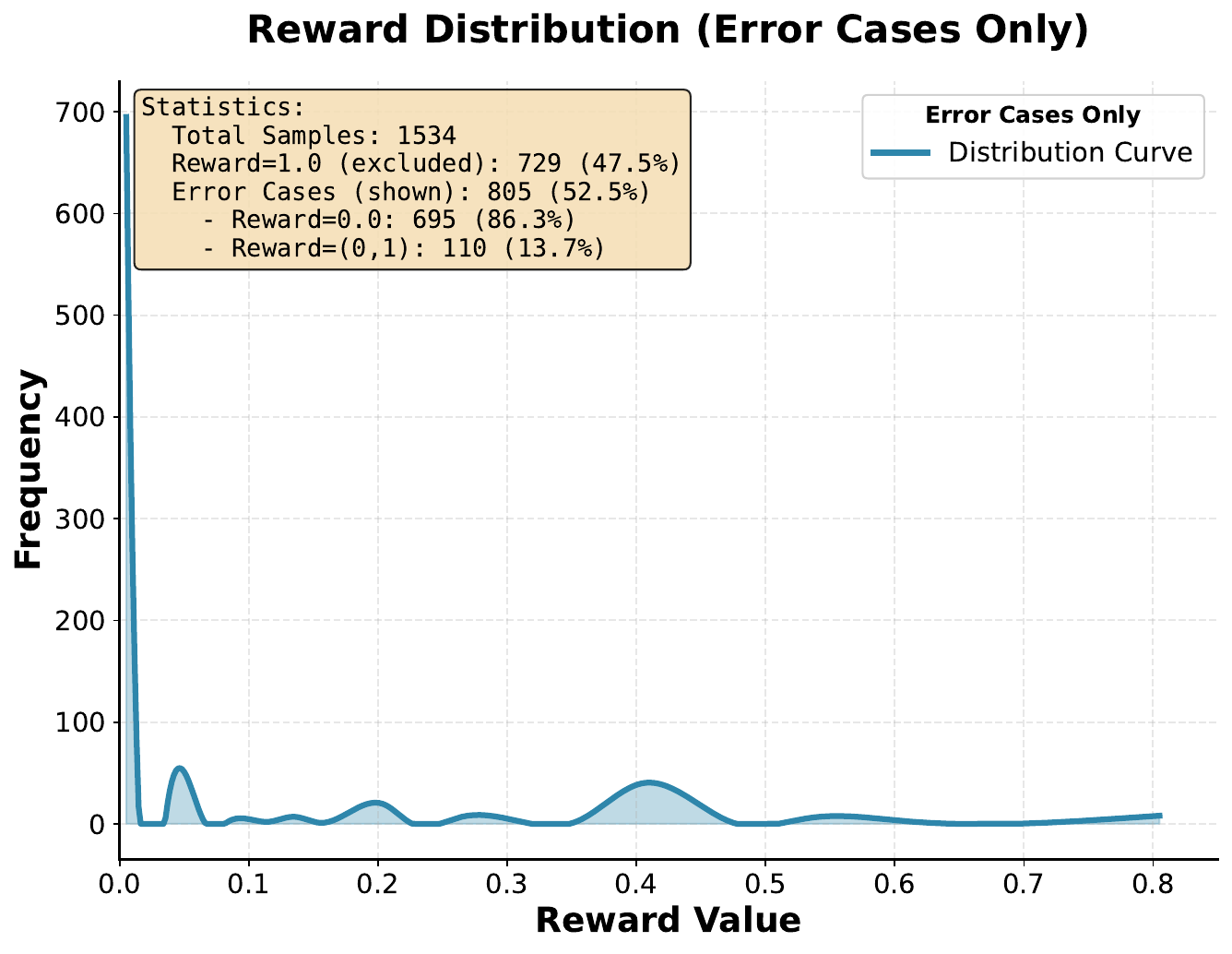}
    \caption{Turn count distribution for models using CSMR on the BIRD-Dev.}
    \label{fig:csmr}
\end{figure}

We quantified the reward distribution of the base model Qwen2.5-7B-Instruct on the Bird dataset. We can clearly see that in Figure~\ref{fig:csmr}, among the 52.5\% of data points that are not completely correct, 13.7\% of the data contains more information, as its response is partially correct. However, the traditional 0/1 distribution completely missed this information, which the CSMR mechanism can capture.

\begin{figure*}[htb!]
    \centering
    \includegraphics[width=0.9\textwidth]{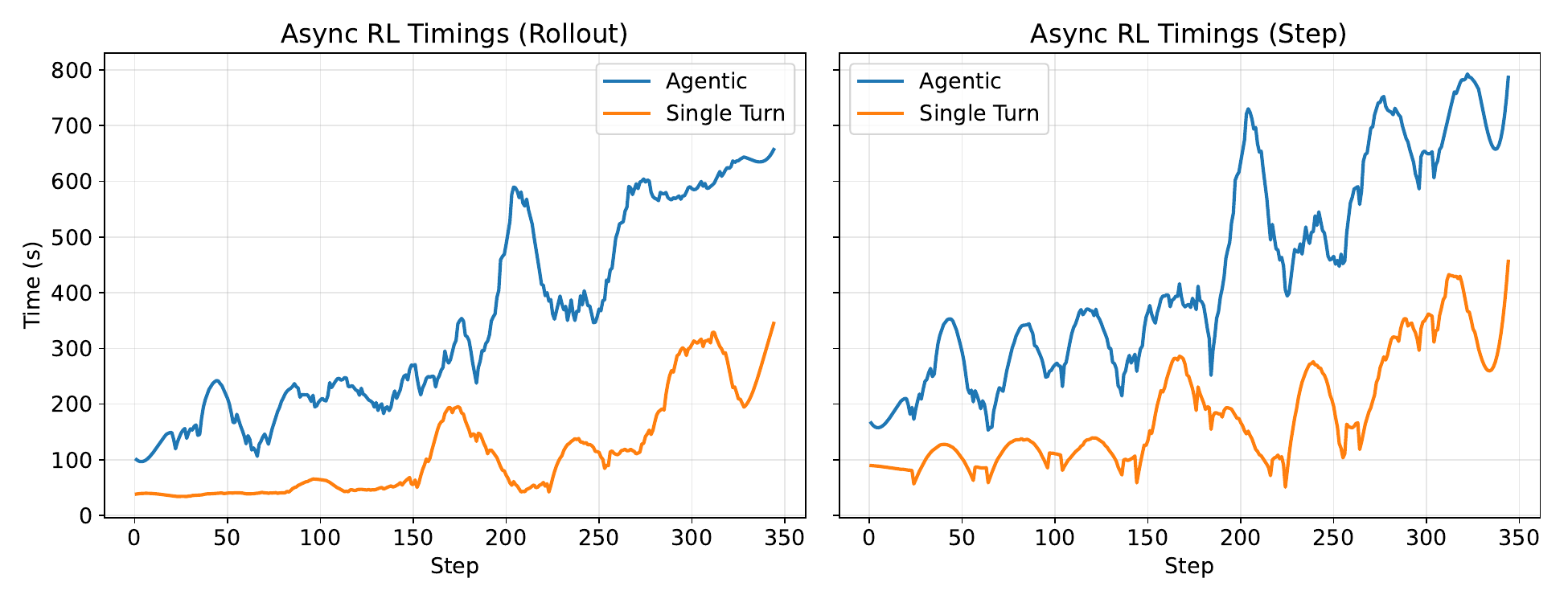} 
    \caption{Agentic vs. Single-Turn training speed comparison.}
    \label{fig:efficiency_comparison}
\end{figure*}
\subsection{The Role of ATR}
Effective reward design hinges on a clear \textbf{reward shape}, enabling the model to gauge performance and optimize strategically. We theoretically examine ATR’s guidance toward optimal trajectories, demonstrated under two distinct signals in Figure~\ref{ATR}. 

Under the \textbf{binary signal}, we prioritize reaching the correct solution efficiently (i.e., fewer steps or simpler logic); here, a persistent incorrect trajectory (e.g., \texttt{[0, 0]}) receives a $-0.10$ baseline, which is significantly higher than the reward for a \textbf{regressive trajectory} (e.g., \texttt{[1, 0]}, $-0.60$). This disparity penalizes abandoning validated correct states, effectively suppressing policy oscillations and enforcing stability across turns. Conversely, under the \textbf{dense signal}, we analyze trajectories within a three-turn window, rewarding their morphological ``shape'' (e.g., monotonic increase, rise-then-fall, or stagnation). Results show that \textbf{``Monotonically Increasing''} trajectories receive the highest average reward, encouraging beneficial exploration toward convergence. Crucially, high-quality \textbf{Dense Rewards} are indispensable for internalizing gradual refinement, underscoring the synergy between \textbf{CSMR} (providing high-fidelity signals) and \textbf{ATR} (integrating signals into a coherent trajectory evaluation).

\subsection{Training Efficiency Comparison}

We use the async setup in Verl, where, in each step, it connects vLLM inference and tool calls to execute a unified rollout, and then calculates the $\mathbf{\log p}$ to update the gradients. Therefore, we compared the rollout time and the total time for the Agentic and single-turn methods at each step, as shown in Figure~\ref{fig:efficiency_comparison}. The execution time for the Agentic rollout is almost twice that of the single-turn method. The rollout phase accounts for the longest duration. The update phase is also longer for Agentic because it involves a longer response, resulting in a total time that is twice that of the single-turn method.

\section{Related Work}
\paragraph{RLVR for Single-Turn Text-to-SQL}
A more recent development is RLVR~\citep{yue2025does}, a promising strategy that is part of a broader trend of boosting LLM reasoning (e.g., OpenAI’s O-series~\cite{openai2025o3}, DeepSeek-R1~\cite{guo2025deepseek}, and Kimi K1.5~\cite{kimi2025_k1_5}), especially in areas like mathematics, coding, and reasoning~\citep{deepseekmath, qwen2.5, htl}. Recent reasoning-oriented models have primarily targeted single-turn Text-to-SQL tasks. STaR-SQL~\cite{he2025starsqlselftaughtreasonertexttosql} utilizes rationale-based SFT. Reasoning-SQL~\cite{pourreza2025reasoning} and SQL-R1~\cite{ma2025sql} apply reinforcement learning to ensure logical and execution consistency; notably, Reasoning-SQL~\cite{pourreza2025reasoning} also designed an n-gram-based reward mechanism to provide more fine-grained feedback. These methods fail to surpass single-turn bottlenecks due to a lack of interactive environment verification.

\paragraph{Multi-turn Agent RL}
Early multi-turn LLM methods relied heavily on prompting: ACT-SQL~\cite{zhang2023actsql} rewrites multi-turn queries into single-turn inputs via Chain-of-Thought prompting, while CoE-SQL~\cite{zhang2024coesql} incrementally edits the prior SQL. Both of these methods depend on closed-source GPT models and lack database verification or self-correction mechanisms.
In contrast, Agent models like Search-R1~\cite{jin2025_searchR1}, WebAgent-R1~\cite{wei2025_webagentR1}, verltool~\cite{jiang2025verltoolholisticagenticreinforcement} and Retool~\cite{feng2025retool} extend reasoning capabilities through multi-turn interactions with an environment. Applying this paradigm to the SQL domain, MTSQL-R1~\cite{guo2025mtsql} has adopted multi-turn execution to call SQL statements, using reinforcement learning to train the model's interaction with the external database.
However, these reinforcement learning methods mostly rely on a binary reward based on the final answer. In a long-horizon, multi-turn trajectory, this binary reward signal is extremely sparse, making it difficult for the model to effectively distinguish the quality of the reasoning process~\cite{huo2025bird,cui2025processreinforcementimplicitrewards,Fine-Grained_Human_Feedback,Luo2024ImproveMR}.

\section{Conclusion}

In this paper, we addressed the critical bottleneck of sparse-feedback-driven credit assignment in Agentic RL, focusing our experimental validation on the Text-to-SQL task domain. We proposed Agentic SQL along with a universal two-tiered reward mechanism, which not only resolves the challenges associated with sparse reward signals but also advances the SOTA in Text-to-SQL generation. A key contribution of this work is the integration of \textit{Lyapunov stability theory} into the reinforcement learning reward design. We rigorously proved that our asymmetric transition matrix acts as an energy dissipation operator, mathematically guaranteeing a cycle-free policy and ensuring asymptotic stability throughout the reasoning process. Detailed ablation experiments highlight the specific contribution of each design component. These results validate that dense, process-oriented signals are essential for bridging the gap between LLMs reasoning and real-world database interactions.
\section*{Limitations}

Although our framework significantly improves the performance of agentic SQL, several limitations remain to be addressed in future work:

\paragraph{Computational Overhead and Latency} The transition from a single-turn paradigm to a multi-turn agentic framework inevitably increases computational costs. Our experimental analysis shows that the execution time for the agentic rollout is nearly twice that of the single-turn method, with the rollout phase accounting for the longest duration. This increased latency, combined with longer response lengths during the update phase, poses challenges for real-time deployment in latency-sensitive environments.

\paragraph{Fixed Interaction Horizon} Our current implementation adopts a finite-horizon MDP with a system prompt that limits the agent to a maximum of three tool calls. While this prevents infinite loops and controls costs, it may restrict the model's ability to solve extremely complex, multi-step reasoning tasks in enterprise-grade workflows that require more extensive exploration.

\paragraph{Hyperparameter Dependency in Reward Design} While we leverage Lyapunov stability theory to provide theoretical guarantees for convergence, the practical implementation still relies on several heuristic hyperparameters. These parameters include the stagnation threshold $\tau$ within the ATR mechanism. Although our ablation experiments demonstrate that the system exhibits a certain degree of robustness, these parameters may require further fine-tuning when the framework is applied to database domains with significantly different architectural complexities or data distributions.
\bibliography{custom}

\appendix
\newpage
\section{Key Components of the MDP}

\paragraph{State Space ($\mathcal{S}$):} The state $s_t \in \mathcal{S}$ encapsulates the complete interaction history, including the original natural language question $q$, the database schema $\mathcal{D}_{\text{schema}}$, and the sequence of past actions and observations.

\paragraph{Action Space ($\mathcal{A}$):} The structured tool-use signals parsed from the complete response generated by the policy $\pi_\theta$. Primary tools include \texttt{execute\_sql(query: str)} and \texttt{finish(answer: str)}.

\paragraph{Transition Dynamics ($\mathcal{P}$) and Rethink Mechanism:} The dynamics are determined by the agent-environment interaction and, crucially, a rethink mechanism activated upon receiving the database observation $o_t$. This mechanism enables the agent to autonomously evaluate the result and decide whether to continue the iteration or terminate.

\paragraph{Process-based Reward Function ($\mathcal{R}$):} Unlike other agent frameworks, the core of our approach is a dense, \textbf{process-based reward function}, $\mathcal{R}(s_t, a_t)$. This function has two distinguishing properties: first, it provides fine-grained feedback for scoring the outcome, rather than a single binary reward; second, we introduce a fused, composite reward signal that offers a fine-grained reflection of our trajectory's intermediate steps, rather than judging based solely on the final result. This approach effectively mitigates the credit assignment problem inherent in sparse, outcome-only rewards.

\section{Proof of Theorem 1 (Strict Dissipativity)}
\label{app:proof}

To theoretically justify how ATR prunes limit cycles, we model the agent's reasoning as a discrete-time dynamical system on a semantic state space $\mathcal{S}$.

\begin{definition}[Lagrange Stability]
A reasoning system is \textbf{Lagrange stable} if every trajectory $\tau = \{s_1, s_2, \dots\}$ remains within a bounded subset of $\mathcal{S}$. In the context of Agentic RL, this represents a state where the agent oscillates indefinitely (limit cycles) without diverging or converging.
\end{definition}

\begin{definition}[Strict Dissipativity]
The system satisfies the \textbf{Strict Dissipativity Condition} if for any non-trivial closed cycle $\mathcal{C}$, the accumulated reward is strictly negative: $\sum_{t \in \mathcal{C}} R_t < 0$. This implies a net loss of ``energy'' over any repetitive path, preventing the agent from remaining in a loop.
\end{definition}

\begin{theorem}[Elimination of Limit Cycles]
The optimal policy $\pi^*$ is cycle-free if and only if the reward transition function satisfies the strict dissipativity condition.
\end{theorem}

\begin{proof}
Consider a potential limit cycle $\mathcal{C}$ consisting of a \textbf{regression} step $s_{\text{High}} \to s_{\text{Low}}$ (with reward $R_{\text{High} \to \text{Low}}$) and a \textbf{restoration} step $s_{\text{Low}} \to s_{\text{High}}$ (with reward $R_{\text{Low} \to \text{High}}$).

\begin{enumerate}
    \item \textbf{Symmetric Case:} If the reward scheme is symmetric, $|R_{\text{High} \to \text{Low}}| = |R_{\text{Low} \to \text{High}}|$, then $\sum_{t \in \mathcal{C}} R_t = 0$. The system is Lagrange stable but lacks a driving force to exit the loop.
    
    \item \textbf{Asymmetric Case (ATR):} Our ATR design follows the matrix $\mathcal{M}$, imposing $|R_{\text{High} \to \text{Low}}| > |R_{\text{Low} \to \text{High}}|$. Let $R_{\text{High} \to \text{Low}} = -k \cdot r$ and $R_{\text{Low} \to \text{High}} = r$ where $k > 1$. The total cycle return becomes:
    \begin{equation}
        \sum_{t \in \mathcal{C}} R_t = (1-k) \sum r < 0
    \end{equation}
    
    \item \textbf{Conclusion:} Per Lyapunov stability theory, since every cycle incurs a net cost, the system acts as an \textbf{energy dissipation operator}. This net loss forces the trajectory to terminate at the global maximum, effectively pruning all limit cycles from the policy space. \qedhere
\end{enumerate}
\end{proof}

\section{Training Detail}
\label{append:parameters}
\begin{table}[ht]
    \centering
    
    \caption{Hyperparameters for RL Training}
    \label{parameter_table}
    \small 
    \begin{tabular}{c|c}
        \toprule
        \textbf{Hyperparameter} & \textbf{Value} \\
        \midrule
        Batch Size              & 128 \\
        Learning Rate           & 1e-6 \\ 
        Rollout Temperature     & 1.0 \\
        Rollout Top-p           & 0.95 \\
        Validation Temperature  & 0.6 / 1.0 \\
        Validation Top-p        & 0.95 \\
        PPO Epochs              & 1 \\
        Max Response Length     & 2048 / 4096 \\
        Number of Rollouts      & 8 \\
        Training Epochs         & 5 \\
        GPU Con & 32 * A800 \\
        Training Data & 8,958/17614 \\
        B & 2 \\
        $\alpha$ & 0.8 \\
        $\mathcal{C}_{\text{turn}}$ & 0.0001 \\
        $\tau$ & 0.6 \\
        
        \bottomrule
    \end{tabular}
\end{table}

We utilize the \textbf{BIRD} training set to train our models. This dataset comprises 9,428 question-SQL pairs from 70 databases across diverse domains such as airlines, movies, and sales. We filtered out samples where the gold SQL query failed to execute, resulting in a final set of 8,958 training instances. For evaluation, we primarily use the \textbf{BIRD-Dev} set to assess in-domain capabilities. We introduce the \textbf{Spider} benchmark to test the model's out-of-distribution (OOD) generalization. Additionally, we leverage the challenging enterprise-grade \textbf{Spider-2.0} dataset to evaluate the model's \textit{Agent} capabilities. This dataset targets highly difficult real-world scenarios, representing a strong indicator of practical Text-to-SQL deployment readiness. Specifically, we extract a subset of 135 locally executable examples from Spider-2.0-Lite. To ensure a fair comparison, we evaluate both SQL-R1 and Arctic-Text2SQL-R1-7B using the identical configuration. Throughout all evaluations, we adopt a 0/1 reward mechanism, where a query is considered correct only if all rows in the execution result match the ground truth exactly.

Our experiments employed two sets of trials. The first set involved the Qwen2.5-7B-Instruct model, which underwent Reinforcement Learning (RL) training directly without cold start. The second set involved the OmniSQL model, which required the Format-6k fine-tuning step to teach it the tool-calling format. Format-6k denotes that a subset of data was randomly sampled from the training set and formatted into the tool-calling format. 

During our experimentation, we observed that general-purpose code models, such as Qwen2.5-Coder~\cite{hui2024qwen25codertechnicalreport}, are overly specialized. This specialization leads to poor \textit{instruction-following} performance, rendering them unable to correctly learn the required Agent format. Even when SFT techniques were applied to correct their output format, the model failed to exhibit strong exploratory capabilities in subsequent multi-turn generation tasks. 

For the RL framework, we utilize the open-source library VERL; for the Spider 2.0 evaluation, we employ the open-source repository ReFoRCE~\cite{deng2025reforce}.

The hyperparameter configurations are detailed in Table~\ref{parameter_table}. A single Reinforcement Learning experiment for Agentic SQL takes approximately 24 hours to complete. Specifically, $\alpha=0.8$ serves as the scaling factor for the CSMR to cap structural matches at 0.8, effectively distinguishing them from truly perfect matches. The parameter $B=2$ defines the reward clipping boundary to constrain the total reward within $[-2, 2]$ for training stability. Additionally, $C_{\text{turn}}=0.0001$ introduces a turn penalty to incentivize concise reasoning, while $\tau=0.6$ acts as a stagnation threshold to determine state transitions for the ATR calculation.

\section{Turn Correct Distribution}

\begin{figure}[h]
    \centering
    \includegraphics[width=0.99\columnwidth]{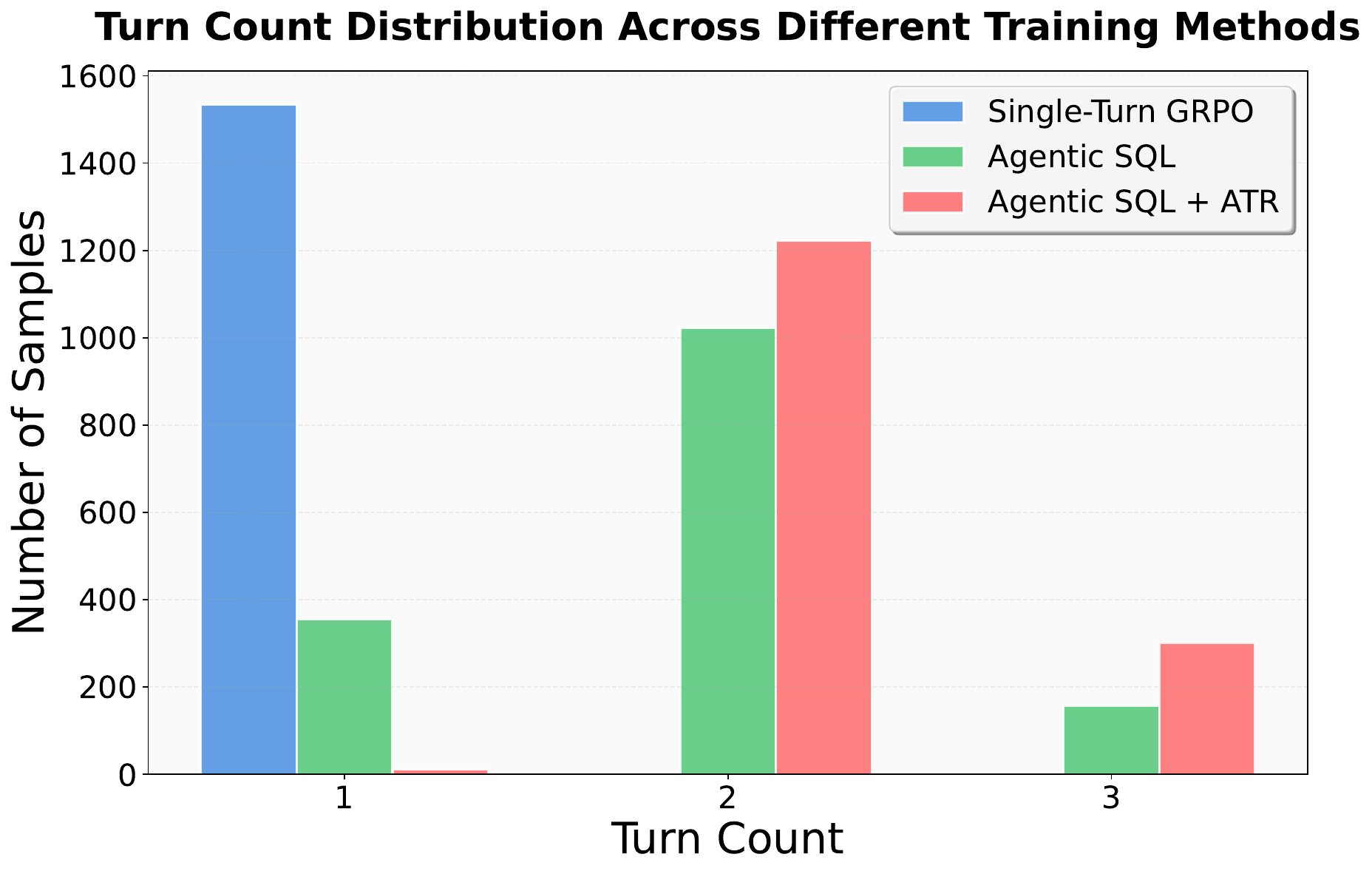}
    \caption{Turn count distribution for models using CSMR on the BIRD-Dev.}
    \label{fig:adaptive}
\end{figure}

\begin{table}[h]
    \centering
    \small
    \caption{Model Performance on Turn-n Interactions.}
    \begin{tabular}{l c c c}
        \toprule
        Model & Turn 1 & Turn 2 & Turn 3 \\
        \midrule
        Single-Turn GRPO & 59.4 & - & - \\
        Agentic SQL & 62.2 & 66.0 & 39.5 \\
        Agentic SQL + ATR & 72.7 & 71.0 & 36.2 \\
        \bottomrule
    \end{tabular}
    \label{tab:mic}
\end{table}

Figure~\ref{fig:adaptive} presents the turn-number distribution (i.e., the percentage of generations using 2 turns, 3 turns, etc.) for the RL-trained models. Conversely, Table~\ref{tab:mic} reports the accuracy achieved by these models at a specified number of $N$ turns.

We observe that different reward functions result in markedly distinct distributions. The model trained without ATR exhibits a more balanced distribution, generating outputs across 1, 2, and 3 turns. In contrast, the model employing ATR almost entirely abandons single-turn generation, instead focusing heavily on 2-turn and 3-turn generations, where it simultaneously achieves the highest overall accuracy.

\section{Additional Experiment}
For experimental simplicity, we conduct these comparative trials using a single-turn GRPO baseline. The primary objective of $\alpha$ is to provide a clear distinction between truly perfect samples and ``pseudo-perfect'' matches---cases where the column value-sets match but the row-level compositions are incorrect, such as $P = \texttt{[[a1, b2], [a2, b1]]}$ and $G = \texttt{[[a1, b1], [a2, b2]]}$. By incorporating $\alpha$, the model effectively learns to differentiate between these two scenarios, thereby validating the efficacy of our approach. While ``pseudo-perfect'' matches may be introduced, this mechanism ensures they are appropriately penalized to avoid misleading the learning process.

\begin{table}[h]
    \centering
    \small 
    \caption{Sensitivity Analysis of the Scaling Factor $\alpha$ in CSMR.}
    \begin{tabular}{l c}
        \toprule
        Model & Bird-Dev \\
        \midrule
        Single-Turn GRPO & \\
        \quad $\alpha=0.6$ & 59.2 \\
        \quad $\alpha=0.7$ & 60.1 \\
        \quad $\alpha=0.8$ & 59.4 \\
        \bottomrule
    \end{tabular}
    \label{tab:alpha_aba}
\end{table}

\section{Case Study}
\label{sec:case study}
In this section, we present a complete multi-turn case study. It illustrates how the agent uses information about missing columns returned by the executor to iteratively refine its plan, ultimately leading to the correct result.

\begin{dialogbox}
\begingroup
\colorbox{yellow!3}{System Prompt} \\
\begin{lstlisting}[]
You are a helpful SQL assistant. You will solve the user's question by calling a tool to execute SQL code and query a SQL database. After each tool call, you must evaluate the result returned in <tool_response></tool_response> to determine its relevance to the question, and decide for yourself whether to make another query. You may also use a tool to inspect the database schema first before providing your final SQL code. You can call the tool a maximum of 3 times.
\end{lstlisting}
\colorbox{yellow!3}{User Prompt} \\
\begin{lstlisting}[language=SQL]
Task Overview:
You are a data science expert. Below, you are provided with a database schema and a natural language question. Your task is to understand the schema and generate a valid SQL query to answer the question.
Database Engine:
SQLite
Database Schema:
CREATE TABLE frpm (
    CDSCode text, -- example: ['01100170109835', '01100170112607']
    `Academic Year` text, -- example: ['2014-2015']
    `County Code` text, -- example: ['01', '02']
    `District Code` integer, -- example: [10017, 31609]
    `School Code` text, -- example: ['0109835', '0112607']
    `County Name` text, -- example: ['Alameda', 'Alpine']
    `District Name` text, -- example: ['Alameda County Office of Education', 'California School for the Blind (State S']
    `School Name` text, -- example: ['FAME Public Charter', 'Envision Academy for Arts & Technology']
    `District Type` text, -- example: ['County Office of Education (COE)', 'State Special Schools']
    `School Type` text, -- example: ['K-12 Schools (Public)', 'High Schools (Public)']
    `Educational Option Type` text, -- example: ['Traditional', 'Juvenile Court School']
    `NSLP Provision Status` text, -- example: ['Breakfast Provision 2', 'Provision 2']
    `Charter School (Y/N)` integer, -- example: [1, 0]
    `Charter School Number` text, -- example: ['0728', '0811']
    `Charter Funding Type` text, -- example: ['Directly funded', 'Locally funded']
    IRC integer, -- example: [1, 0]
    `Low Grade` text, -- example: ['K', '9']
    `High Grade` text, -- example: ['12', '8']
    `Enrollment (K-12)` real, -- example: [1087.0, 395.0]
    `Free Meal Count (K-12)` real, -- example: [565.0, 186.0]
    `Percent (%) Eligible Free (K-12)` real, -- example: [0.5197, 0.4708]
    `FRPM Count (K-12)` real, -- example: [715.0, 186.0]
    `Percent (%) Eligible FRPM (K-12)` real, -- example: [0.6577, 0.4708]
    `Enrollment (Ages 5-17)` real, -- example: [1070.0, 376.0]
    `Free Meal Count (Ages 5-17)` real, -- example: [553.0, 182.0]
    `Percent (%) Eligible Free (Ages 5-17)` real, -- example: [0.5168, 0.4840]
    `FRPM Count (Ages 5-17)` real, -- example: [702.0, 182.0]
    `Percent (%) Eligible FRPM (Ages 5-17)` real, -- example: [0.6560, 0.4840]
    `2013-14 CALPADS Fall 1 Certification Status` integer, -- example: [1]
    PRIMARY KEY (CDSCode),
    CONSTRAINT fk_frpm_cdscode FOREIGN KEY (CDSCode) REFERENCES schools (CDSCode)
);
-- (Other tables satscores and schools are omitted for brevity, but follow same pattern)


CREATE TABLE satscores (
    cds text, -- example: ['10101080000000', '10101080109991']
    rtype text, -- example: ['D', 'S']
    sname text, -- school name, example: ['FAME Public Charter', 'Envision Academy for Arts & Technology']
    dname text, -- district name, example: ['Alameda County Office of Education', 'Alameda Unified']
    cname text, -- county name, example: ['Alameda', 'Amador']
    enroll12 integer, -- enrollment (1st-12nd grade), example: [398, 62]
    NumTstTakr integer, -- Number of Test Takers, example: [88, 17]
    AvgScrRead integer, -- average scores in Reading, example: [418, 503]
    AvgScrMath integer, -- average scores in Math, example: [418, 546]
    AvgScrWrite integer, -- average scores in writing, example: [417, 505]
    NumGE1500 integer, -- Number of Test Takers Whose Total SAT Scores Are Greater or Equal to 1500, example: [14, 9]
    PRIMARY KEY (cds),
    CONSTRAINT fk_satscores_cds FOREIGN KEY (cds) REFERENCES schools (CDSCode)
);

CREATE TABLE schools (
    CDSCode text, -- example: ['01100170000000', '01100170109835']
    NCESDist text, -- National Center for Educational Statistics school district identification number, example: ['0691051', '0600002']
    NCESSchool text, -- National Center for Educational Statistics school identification number, example: ['10546', '10947']
    StatusType text, -- example: ['Active', 'Closed']
    County text, -- example: ['Alameda', 'Alpine']
    District text, -- example: ['Alameda County Office of Education', 'California School for the Blind (State S']
    School text, -- example: ['FAME Public Charter', 'Envision Academy for Arts & Technology']
    Street text, -- example: ['313 West Winton Avenue', '39899 Balentine Drive, Suite 335']
    StreetAbr text, -- street address, example: ['313 West Winton Ave.', '39899 Balentine Dr., Ste. 335']
    City text, -- example: ['Hayward', 'Newark']
    Zip text, -- example: ['94544-1136', '94560-5359']
    State text, -- example: ['CA']
    MailStreet text, -- example: ['313 West Winton Avenue', '39899 Balentine Drive, Suite 335']
    MailStrAbr text, -- mailing street address, example: ['313 West Winton Ave.', '39899 Balentine Dr., Ste. 335']
    MailCity text, -- mailing city, example: ['Hayward', 'Newark']
    MailZip text, -- mailing zip, example: ['94544-1136', '94560-5359']
    MailState text, -- mailing state, example: ['CA']
    Phone text, -- example: ['(510) 887-0152', '(510) 596-8901']
    Ext text, -- extension, example: ['130', '1240']
    Website text, -- example: ['www.acoe.org', 'www.envisionacademy.org/']
    OpenDate date, -- example: ['2005-08-29', '2006-08-28']
    ClosedDate date, -- example: ['2015-07-31', '2015-06-30']
    Charter integer, -- example: [1, 0]
    CharterNum text, -- example: ['0728', '0811']
    FundingType text, -- example: ['Directly funded', 'Locally funded']
    DOC text, -- District Ownership Code, example: ['00', '31']
    DOCType text, -- The District Ownership Code Type, example: ['County Office of Education (COE)', 'State Special Schools']
    SOC text, -- School Ownership Code, example: ['65', '66']
    SOCType text, -- School Ownership Code Type, example: ['K-12 Schools (Public)', 'High Schools (Public)']
    EdOpsCode text, -- Education Option Code, example: ['TRAD', 'JUV']
    EdOpsName text, -- Educational Option Name, example: ['Traditional', 'Juvenile Court School']
    EILCode text, -- Educational Instruction Level Code, example: ['ELEMHIGH', 'HS']
    EILName text, -- Educational Instruction Level Name, example: ['Elementary-High Combination', 'High School']
    GSoffered text, -- grade span offered, example: ['K-12', '9-12']
    GSserved text, -- grade span served., example: ['K-12', '9-12']
    Virtual text, -- example: ['P', 'N']
    Magnet integer, -- example: [0, 1]
    Latitude real, -- example: [37.658212, 37.521436]
    Longitude real, -- example: [-122.09713, -121.99391]
    AdmFName1 text, -- administrator's first name 1, example: ['L Karen', 'Laura']
    AdmLName1 text, -- administrator's last name 1, example: ['Monroe', 'Robell']
    AdmEmail1 text, -- administrator's email address 1, example: ['lkmonroe@acoe.org', 'laura@envisionacademy.org']
    AdmFName2 text, -- administrator's first name 2, example: ['Sau-Lim (Lance)', 'Jennifer']
    AdmLName2 text, -- administrator's last name 2, example: ['Tsang', 'Koelling']
    AdmEmail2 text, -- administrator's email address 2, example: ['stsang@unityhigh.org', 'jkoelling@efcps.net']
    AdmFName3 text, -- administrator's first name 3, example: ['Drew', 'Irma']
    AdmLName3 text, -- administrator's last name 3, example: ['Sarratore', 'Munoz']
    AdmEmail3 text, -- administrator's email address 3, example: ['dsarratore@vincentacademy.org', 'gmunoz@piedmont.k12.ca.us']
    LastUpdate date, -- example: ['2015-06-23', '2015-09-01']
    PRIMARY KEY (CDSCode)
);
This schema describes the database's structure, including tables, columns, primary keys, foreign keys, and any relevant relationships or constraints.
Question:
Charter schools refers to `Charter School (Y/N)` = 1 in the frpm
Please list the phone numbers of the direct charter-funded schools that are opened after 2000/1/1.

Instructions:
- Make sure you only output the information that is asked in the question. If the question asks for a specific column, make sure to only include that column in the SELECT clause, nothing more.
- The generated query should return all of the information asked in the question without any missing or extra information.
- Before generating the final SQL query, please think through the steps of how to write the query.
Output Format:
In your answer, please enclose the generated SQL query in a code block:
```sql
-- Your SQL query
```
Take a deep breath and think step by step to find the correct SQL query.
For each round's response, please first think step-by-step, and then make the tool call according to this format: 
<tool_call> 
{
    "name": "run_sql_remote",
    "arguments": {
        "query": "```sql\n-- Your SQL query\n```",
        "db_id": "california_schools/california_schools.sqlite"
    }
}
</tool_call>
\end{lstlisting}
\colorbox{yellow!3}{Response-Turn 1}\\
\begin{lstlisting}[language=SQL]
To find the phone numbers of the direct charter-funded schools that are opened after 2000/1/1, we need to follow these steps:

1. Filter the schools to include only those with `Charter School (Y/N)` = 1 (charter schools).
2. Ensure the schools are direct charter-funded, which can be checked from the `Charter Funding Type` column in the `frpm` table.
3. Verify that the schools were opened after 2000/1/1 using the `OpenDate` column in the `schools` table.
Let's write the SQL query step by step:
1. Join the `frpm` table with the `schools` table using the `CDSCode` column.
2. Apply the filter for charter schools (`Charter School (Y/N)` = 1).
3. Check for direct charter funding (`Charter Funding Type` = 'Directly funded').
4. Ensure the schools were opened after 2000/1/1 (`OpenDate` > '2000-01-01').
5. Select the phone numbers from the `schools` table.
Now, let's write the SQL query:
<tool_call>
{
    "name": "run_sql_remote",
    "arguments": {
        "query": "```sql\nSELECT s.Phone\nFROM schools s\nJOIN frpm f ON s.CDSCode = f.CDSCode\nWHERE f.`Charter School (Y/N)` = 1 AND s.`Charter Funding Type` = 'Directly funded' AND s.OpenDate > '2000-01-01'\n```",
        "db_id": "california_schools/california_schools.sqlite"
    }
}
</tool_call>
\end{lstlisting}
\colorbox{yellow!3}{SQL Executor Interaction Results} \\
\begin{lstlisting}[language=SQL]
<tool_response>
error:no such column: s.Charter Funding Type
</tool_response>
\end{lstlisting}
\colorbox{yellow!3}{Response-Turn 2} \\
\begin{lstlisting}[language=SQL]
I apologize for the oversight. The `Charter Funding Type` column is actually in the `frpm` table, not the `schools` table. Therefore, we should join the `frpm` table with the `schools` table and apply the filters using the columns from the `frpm` table.

Let's correct the query:

1. Join the `frpm` table with the `schools` table using the `CDSCode` column.
2. Apply the filter for charter schools (`Charter School (Y/N)` = 1).
3. Check for direct charter funding (`Charter Funding Type` = 'Directly funded') in the `frpm` table.
4. Ensure the schools were opened after 2000/1/1 (`OpenDate` > '2000-01-01').
5. Select the phone numbers from the `schools` table.

Now, let's write the corrected SQL query:
<tool_call>
{
    "name": "run_sql_remote",
    "arguments": {
        "query": "```sql\nSELECT s.Phone\nFROM schools s\nJOIN frpm f ON s.CDSCode = f.CDSCode\nWHERE f.`Charter School (Y/N)` = 1 AND f.`Charter Funding Type` = 'Directly funded' AND s.OpenDate > '2000-01-01'\n```",
        "db_id": "california_schools/california_schools.sqlite"
    }
}
</tool_call>
\end{lstlisting}
\colorbox{yellow!3}{SQL Executor Interaction Results} \\
\begin{lstlisting}[language=SQL]
<tool_response>
[[None], ['(510) 596-8901'], [None], ['(510) 686-4131'], ['(510) 452-2063'], ['(510) 842-1181'], ['(510) 748-4008'], ['(510) 748-4017'], ['(510) 995-4300'], ['(510) 748-4314'], ['(510) 809-9800'], ['(510) 809-9800'], ['(510) 300-1340'], ['(510) 300-1560'],
</tool_response>
\end{lstlisting}
\colorbox{yellow!3}{Response-Turn 2} \\
\begin{lstlisting}[language=SQL]
The phone numbers of the direct charter-funded schools that are opened after 2000/1/1 are as follows:
```sql
SELECT s.Phone
FROM schools s
JOIN frpm f ON s.CDSCode = f.CDSCode
WHERE f.`Charter School (Y/N)` = 1 AND f.`Charter Funding Type` = 'Directly funded' AND s.OpenDate > '2000-01-01'
```
\end{lstlisting}
\endgroup
\end{dialogbox}

\end{document}